\documentclass[runningheads,a4paper]{llncs}

\newcommand{\VECTOR}[1]{\vec{#1}}

\usepackage{graphicx}

\usepackage{url}
\newcommand{\keywords}[1]{\par\addvspace\baselineskip
\noindent\keywordname\enspace\ignorespaces#1}

\usepackage{bbding}
\usepackage{amsmath,amssymb}             
\usepackage{mathrsfs}
\usepackage{chemarrow}
\usepackage{stmaryrd}
\usepackage{ragged2e} 
\usepackage{url}
\usepackage{yfonts}
\usepackage{multirow}
\usepackage{multicol}
\usepackage{ucs}
\usepackage[utf8x]{inputenc}
\usepackage{boxedminipage}
\usepackage{color}
\usepackage{bussproofs}
\usepackage{semantic}
\usepackage{xspace} 
\usepackage{wasysym}
\usepackage[english]{babel}
\usepackage[T1]{fontenc}
\usepackage{keystroke} 
\usepackage{epstopdf}
\usepackage{enumerate}

\usepackage{algorithm}
\usepackage{listings}
\lstdefinelanguage{pseudo}{
  basicstyle=\small,
  numbers=left, 
  stepnumber=1,
  numberstyle=\tiny,
  morekeywords={if, then, else, for, while, from, to, do, done, return,
    function, is, and, or},
  sensitive=true,%
  mathescape=true,
  morecomment=[l]\#,%
  morestring=[b]',%
}

\definecolor{darkgreen}{rgb}{0,0.4,0}
\definecolor{purple}{rgb}{0.4,0,1}
\definecolor{brown}{rgb}{0.6,0.1,0}
\definecolor{white}{rgb}{1,1,1}
\definecolor{black}{rgb}{0,0,0}

\newcommand{\urlsmall}[1]{{\scriptsize\url{#1}}}

\newcommand{\INVISIBLE}[1]{{\textcolor{white}{{#1}}}}
\newcommand{\NOTHING}[0]{\INVISIBLE{\em hack}}

\usepackage{hyperref}
\newcommand{\DEF}[1]{{\sl {#1}}}

\newcommand{\DISJOINTUNION}[0]{\uplus}

\newcommand{\MINF}[0]{-\infty}
\newcommand{\PINF}[0]{+\infty}
\newcommand{\TPLUS}[0]{\oplus}
\newcommand{\TTIMES}[0]{\odot}

\newcommand{\BIGTPLUS}[0]{\displaystyle\bigoplus}
\newcommand{\BIGTTIMES}[0]{\displaystyle\bigodot}
\newcommand{\SBIGTPLUS}[0]{\bigoplus}
\newcommand{\SBIGTTIMES}[0]{\bigodot}
\newcommand{\BIGTPLUSSMALLER}[0]{\bigoplus}
\newcommand{\BIGTTIMESSMALLER}[0]{\bigodot}

\newcommand{\UNION}[0]{\cup}

\newcommand{\UNIVERSE}[0]{\mathbb{U}}
\newcommand{\NATURALS}[0]{\mathbb{N}}
\newcommand{\INTEGERS}[0]{\mathbb{Z}}
\newcommand{\RATIONALS}[0]{\mathbb{Q}}
\newcommand{\TRADITIONALUNIVERSE}[0]{\INTEGERS}
\newcommand{\EQD}[0]{\triangleq}

\newcommand{\BOTTOM}[0]{\bot}
\newcommand{\TURN}[0]{\lambda}
\newcommand{\STURN}[0]{\Lambda}
\newcommand{\PLAYER}[0]{\mathcal{P}}
\newcommand{\OPPONENT}[0]{\mathcal{O}}
\newcommand{\POSITION}[0]{\pi}
\newcommand{\POSITIONS}[0]{\mathbb{P}}
\newcommand{\TERMINALPOSITIONS}[0]{\POSITIONS_{T}}
\newcommand{\ALGEBRA}[0]{\mathcal{A}}
\newcommand{\TO}[0]{\rightarrow}
\newcommand{\FROM}[0]{\leftarrow}
\newcommand{\TOSTAR}[0]{\TO^{*}}
\newcommand{\TOCONTEXT}[0]{\TO_{c}}
\newcommand{\TOCONTEXTSTAR}[0]{\TO_{c}^{*}}
\newcommand{\DOWNTO}[0]{\downarrow}
\newcommand{\DOWNTOSTAR}[0]{\downarrow_{*}}
\newcommand{\PAYOFF}[0]{p}
\newcommand{\OVERPAYOFF}[0]{\bar{p}}
\newcommand{\EVALUATION}[0]{v_p}

\newcommand{\SPLUS}[0]{\sum}
\newcommand{\STIMES}[0]{\prod}

\newcommand{\OPENSEQUENCE}[0]{\langle}
\newcommand{\CLOSESEQUENCE}[0]{\rangle}
\newcommand{\EMPTYSEQUENCE}[0]{\OPENSEQUENCE\CLOSESEQUENCE}
\newcommand{\WHATEVER}[0]{t} 
\newcommand{\TERM}[0]{t}
\newcommand{\TERMS}[0]{{\VECTOR{\TERM}}}
\newcommand{\VALUE}[0]{v}
\newcommand{\VALUES}[0]{{\VECTOR{\VALUE}}}
\newcommand{\IMPLIES}[0]{\Rightarrow}
\newcommand{\SYNTAX}[0]{S}
\newcommand{\REACHABLES}[1]{{#1}\mathord\downarrow}
\newcommand{\WILL}[0]{will}

\newcommand{\TITLE}[0]{How to correctly prune tropical trees}
\newcommand{\GRAMMAR}[0]{\mathscr{G}}

\newcommand{\TALPHA}[0]{$\alpha$}
\newcommand{\TBETA}[0]{$\beta$}
\newcommand{\ALPHABETA}[0]{\TALPHA-\TBETA}
\newcommand{\ONE}[0]{\bf 1}
\newcommand{\ZERO}[0]{\bf 0}

\newcommand{\SHORTTITLE}[0]{\TITLE}
\newcommand{\FINANCEDBY}[0]{This work is partially financed by Marie Curie action n.~29849 Websicola
and ANR-06-JCJC-0122.}

\newcommand{\ARXIV}[1]{#1}
\newcommand{\IFARXIV}[1]{#1}

\begin{document}
\IFARXIV{\linespread{0.925}}{\linespread{0.957}} 

\mainmatter  

\title{\TITLE{\footnotetext{\FINANCEDBY}\ARXIV{\footnotetext{The final publication of this paper is available at \url{www.springerlink.com}.}}}}

\titlerunning{\SHORTTITLE}

%
%
\author{Jean-Vincent Loddo%
\and Luca Saiu}
\authorrunning{Jean-Vincent Loddo and Luca Saiu}

\institute{
Laboratoire d'Informatique de l'Université Paris Nord - UMR 7030\\
Université Paris 13 - CNRS\\
99, avenue Jean-Baptiste Clément - F-93430 Villetaneuse\\
\path|{loddo, saiu}@lipn.univ-paris13.fr|}

%
%

\toctitle{Lecture Notes in Computer Science}
\tocauthor{\TITLE}
\maketitle

\begin{abstract}
We present {\em tropical games}, a generalization of combinatorial
{\em min-max} games based on tropical algebras.
Our model breaks the traditional symmetry of rational zero-sum games
where players have exactly opposed goals ($min$ vs. $max$),
 is
more widely applicable than {\em min-max} and also supports a form of pruning,
despite it being less effective than {\em \ALPHABETA}.
Actually, {\em min-max} games may be seen as particular cases where both the game and
its dual are tropical: when the dual of a tropical game is also tropical, the
power of \ALPHABETA\ is completely recovered.
We formally develop the model and prove that the tropical pruning strategy is
correct, then conclude by showing how the problem of approximated parsing can be
modeled as a tropical game, profiting from pruning.


\end{abstract}

\section{Introduction}
We are all familiar with games such as Chess or Checkers. Such games are
purely \DEF{rational} as they do not involve any element of chance; they are
also \DEF{zero-sum}, as the players' interests are dual: what one ``wins'', the
other ``loses'' --- which is the origin of the $min$-$max$ evaluation
mechanism.
The two fundamental questions to be asked in a rational game are {\em ``Who
  will win?''} and {\em ``How much will she win?''}. Answering such questions
involves \DEF{searching for a strategy} trough a (typically large) \DEF{game tree}.
Some optimized search techniques were developed, which in the case of
combinatorial two-player games include the \DEF{\ALPHABETA\ pruning} technique
\cite{tree-prune,alpha-beta-analysis--knuth}.
\ALPHABETA\ is not an approximated algorithm: its correctness relies on
the mutual distributive properties of $min$ and $max$.
In this work we explore the implications of assuming only {\em one} player to
be rational, breaking the symmetry of the traditional ``double-sided''
rationality. Quite unsurprisingly our \DEF{tropical $\alpha$-pruning} depends
on just {\em one} distributive property,
 a requirement satisfied by {\em tropical
algebras} (Section \ref{alpha-beta-pruning}).

Following the style introduced by
\cite{phd-thesis--loddo} and \cite{alpha-beta-logic--loddo--di-cosmo},
we will distinguish two aspects of two-player combinatorial games: a first
one that we call \DEF{syntactic}, consisting in a description of the possible
game positions and the valid moves leading from a position to another;
the game syntax
 is the formal equivalent of the intuitive notion of the ``game rules''.
By contrast the \DEF{semantic} aspect is concerned about the interpretation of
the game according to the interests of the players, and ultimately about the
answer to the two fundamental questions above.
Our semantics will be based on tropical algebras, and as a consequence our
technique is widely applicable, relying as it does only on their comparatively
weak hypotheses.

We formally define tropical $\alpha$-pruning and prove its soundness, as our
main contribution (Section~\ref{tropical-games}).
A further contribution consists in our formalization of game evaluation and
tropical (and \ALPHABETA) cuts as a small-step semantics, so that proofs can reuse
the results of term-rewriting theory.

Actually, our soundness result subsumes
other works proving \ALPHABETA's soundness over distributive lattices such
as \cite{alpha-beta-logic--loddo--di-cosmo}
and (later) \cite{alpha-beta-pruning-under-partial-orders},
since distributive lattices are bi-tropical structures (Definition~\ref{bi-tropical}).

We conclude by proposing the algorithm design style \DEF{Choose-How-To-Divide and
  Conquer} meant for attacking even apparently unrelated search problems  as
tropical games; we develop approximated parsing as one such problem by showing how it profits from $\alpha$-pruning
(Section~\ref{choose-how-to-divide-and-conquer}).
\section{Combinatorial game syntax and semantics}
\label{syntax-and-semantics}



\subsection{Syntax}
We speak about ``syntax'', hinting at formal grammars, in that some initial
game positions are given, together with some ``rule'' allowing to derive
successive positions from those: in this way a game can be seen as the tree
of all the possibilities of playing it --- the tree of all the possible matches.

\begin{definition}[Syntax]
A game \DEF{syntax} or \DEF{arena} is
a triple $\SYNTAX = (\POSITIONS, \TURN, succ)$, where:
\begin{itemize}
\item
$\POSITIONS$ is the set of all \DEF{game positions}.

\item
the \DEF{turn function} $\TURN : \POSITIONS \TO \{\PLAYER, \OPPONENT\}$,
says whose turn it is: $\PLAYER$ for
``player'' or $\OPPONENT$ for ``opponent''.

\item
the \DEF{successor function} $succ$, taking a game position and returning
all the positions reachable with valid moves from there;
 $succ : \POSITIONS \TO \POSITIONS^{*}$.
\end{itemize}

Given $\SYNTAX = (\POSITIONS, \TURN, succ)$, we define:
\begin{itemize}
\item
the set of \DEF{terminal positions}
$\TERMINALPOSITIONS = \{\POSITION \in \POSITIONS\ |\ succ(\POSITION) = \EMPTYSEQUENCE\}$.

\item
the \DEF{dual arena} $\SYNTAX^{\BOTTOM} = (\POSITIONS, \TURN^{\BOTTOM}, succ)$,
of course with
$\TURN^{\BOTTOM} : \POSITIONS \TO \{\PLAYER, \OPPONENT\}$,
where for any
$\POSITION \in \POSITIONS$ we have
$\TURN^{\BOTTOM}(\POSITION) \neq \TURN(\POSITION)$.

\item
the $move$ relation is the binary version of the $succ$ relation:
for all $\POSITION, \POSITION' \in \POSITIONS$,
$move(\POSITION, \POSITION')$ iff
$\POSITION' = \POSITION_i$ for some $i$, where
$succ(\POSITION) = \OPENSEQUENCE \POSITION_1 ... \POSITION_n \CLOSESEQUENCE$.
\end{itemize}

The arena is called \DEF{alternate-turn} iff
$move(\POSITION, \POSITION')$ implies $\TURN(\POSITION) \neq \TURN(\POSITION')$.

If $move$ is Nötherian we speak about \DEF{Nötherian} or \DEF{finite arena}.

\end{definition}

\begin{remark}[Alternate-turn arenas]
\label{alternate-turn-transform}
It is possible to systematically make a game alternate-turn by ``collapsing''
all the {\em sequences of consecutive moves of the same player} into single moves.
\end{remark}


One of the most important ideas in Game Theory is the \DEF{strategy},
containing a plan to win the game --- a player saying to herself ``if this happens
I should do that, but if this other thing happens I should do that, and so
on''.
It should be noticed that a strategy is only related to the syntactic part of
a game, being independent, {\em per se}, from the game evaluation. In
particular, a strategy may very well not be winning.

\begin{definition}[Strategy]
Let $\SYNTAX = (\POSITIONS, \TURN, succ)$ be an arena, and $\POSITION \in
\POSITIONS$ be a position. We define:
\begin{itemize}
\item
the \DEF{reachable positions from $\POSITION$} as the right elements of the reflexive-transitive
closure of the relation $succ$: $\REACHABLES{\POSITION} = succ^{*}(\POSITION)$;

\item
a \DEF{global strategy} $\sigma$, as a subset of the relation $succ$ which is:
\begin{itemize}
\item \DEF{deterministic} in $\PLAYER$ positions:\\
  for all $\POSITION \in \POSITIONS$ where $\TURN(\POSITION) = \PLAYER$,
  if $succ(\POSITION) = \OPENSEQUENCE \POSITION_1 ... \POSITION_n \CLOSESEQUENCE$
  then
  $\sigma(\POSITION) = \OPENSEQUENCE \POSITION_i \CLOSESEQUENCE$,
  for some $i$ such that $1 \le i \le n$.
\item \DEF{complete} in $\OPPONENT$ positions:\\
  for all $\POSITION \in \POSITIONS$ where $\TURN(\POSITION) = \OPPONENT$,
  $\sigma(\POSITION) = succ(\POSITION)$.
\end{itemize}

\item
a \DEF{strategy for the initial position $\POSITION$} is a global
strategy for the restricted arena
$\SYNTAX_\POSITION = (\REACHABLES{\POSITION}, \TURN|_{\REACHABLES{\POSITION}}, succ|_{\REACHABLES{\POSITION}})$,
where we indicate with $f|_{D}$ the restriction of a function $f$ to the
set $D$.
\end{itemize}
\end{definition}

\subsection{Semantics}
Let us assume a finite game with syntax
$\SYNTAX = (\POSITIONS, \TURN, succ)$. Traditionally the two players have exactly opposed
interests and we assume, by convention, that the player $\PLAYER$ will
try to \DEF{minimize} the \DEF{payoff} of the final position
while the opponent $\OPPONENT$ will try to \DEF{maximize} it.

The ordinary way of {\em evaluating} such a finite game 
consists in labeling non-terminal nodes with the functions $min$ and $max$
(according to the turn), and terminal nodes with the \DEF{payoff} of the
terminal position $\PAYOFF(\POSITION)$.  Such
values are then ``propagated'' back, applying the function at each node to its
children's values. The final value at the root is called the \DEF{game value}:
it says {\em who} wins and {\em how much}, supposing {\em both} players to be
rational.

Hence, assuming $\PAYOFF : \TERMINALPOSITIONS \TO \TRADITIONALUNIVERSE$ in
accord to the tradition, the game value
$v_p : \POSITIONS \TO \TRADITIONALUNIVERSE$ could be simply defined as a function of
the initial position:
\[
\EVALUATION(\POSITION) =
\begin{cases}
\PAYOFF(\POSITION), & \POSITION \in \TERMINALPOSITIONS\\
min_{i=1}^{n}\ \EVALUATION(\POSITION_i), & succ(\POSITION) = \OPENSEQUENCE\pi_1...\pi_n\CLOSESEQUENCE, \TURN(\POSITION) = \PLAYER\\
max_{i=1}^{n}\ \EVALUATION(\POSITION_i), & succ(\POSITION) = \OPENSEQUENCE\pi_1...\pi_n\CLOSESEQUENCE, \TURN(\POSITION) = \OPPONENT
\end{cases}
\]
This classical definition has the obvious defect of only supporting 
the function $min$ and $max$; often for resolving actual games the preferred
structure is $\INTEGERS$, $\RATIONALS$ (possibly extended with $\MINF$ and
$\PINF$), floating point numbers, or some sort of tuples containing such
structures on which a topological order is defined. Hence, in order to be more
general, let us define $\UNIVERSE$ to be any 
set closed over two associative binary operations $\TPLUS$ and $\TTIMES$, where $\TPLUS$
will be associated to the player and $\TTIMES$ to the opponent. Assuming
$p : \POSITIONS \TO \UNIVERSE$, the definition above would become:
\[
\EVALUATION(\POSITION) =
\begin{cases}
\PAYOFF(\POSITION), & \POSITION \in \TERMINALPOSITIONS\\
\BIGTPLUS_{i=1}^{n} \EVALUATION(\POSITION_i), & succ(\POSITION) = \OPENSEQUENCE\pi_1...\pi_n\CLOSESEQUENCE, \TURN(\POSITION) = \PLAYER\\
\BIGTTIMES_{i=1}^{n} \EVALUATION(\POSITION_i), & succ(\POSITION) = \OPENSEQUENCE\pi_1...\pi_n\CLOSESEQUENCE, \TURN(\POSITION) = \OPPONENT
\end{cases}
\]

The extended $\EVALUATION$ above is a step forward, but it still has the
problem of only being well-defined on finite games.
We solve this problem by abandoning the functional definition of
$\EVALUATION$ altogether, and giving a {\em small-step semantics}
instead. Actually, this style will also be useful in
Section~\ref{pruning-soundness-proof} to prove the soundness of our pruning
technique.

\begin{remark}[Invariance under alternate-turn transformation]
\label{alternate-turn-transform-invariance}
It is easy to see that the transformation hinted at in
Remark~\ref{alternate-turn-transform} does not alter semantics, because of
the two associative properties.
\end{remark}

\begin{definition}[Game]
A \DEF{game} is
the triple $G = (\SYNTAX, \ALGEBRA, \PAYOFF)$, where
$\SYNTAX = (\POSITIONS, \TURN, succ)$ is the syntax,
$\ALGEBRA = (\UNIVERSE, \TPLUS, \TTIMES)$ is an algebra with associative
operations $\TPLUS$ and $\TTIMES$, and where
$p : \TERMINALPOSITIONS \TO \UNIVERSE$ is the payoff function.
\end{definition}

Sometimes we informally refer to syntactic or semantic properties as if they
belonged to a game, for example by speaking about ``Nötherian game'' instead of
``Game with Nötherian syntax''.

\subsubsection{Small-step operational semantics}
In the following, we assume a game $G = (\SYNTAX, \ALGEBRA, \PAYOFF)$, where
$\SYNTAX = (\POSITIONS, \TURN, succ)$
and
$\ALGEBRA = (\UNIVERSE, \TPLUS, \TTIMES)$.
\\
The configurations of our system consist of (ground) \DEF{terms} of $G$, recursively
defined as:
$Ter(G) =
 \POSITIONS
 \DISJOINTUNION
 \UNIVERSE
 \DISJOINTUNION
 (\{\SPLUS, \STIMES\} \times Ter(G)^{+})$:

\begin{itemize}
\item positions in $\POSITIONS$ indicate game positions still to be expanded (if not terminal)
  and evaluated (otherwise).

\item values in $\UNIVERSE$ denote the value, already fully computed, of some sub-terms.

\item
a complex term such as $\SPLUS \OPENSEQUENCE t_1 ... t_n \CLOSESEQUENCE$
or
$\STIMES \OPENSEQUENCE t_1 ... t_n \CLOSESEQUENCE$
indicates a position 
at some state of its evaluation; $\SPLUS$ or $\STIMES$ holding the turn
information, and $t_1 ... t_n$ representing the game subterms from
that state on.
\end{itemize}
It is crucial not to mistake {\em terms of $G$}, which represent partially expanded game
trees, for {\em game positions}, which in practice will also tend to be structured
symbolic terms, but can be considered atomic at a high level: the rewrite rules shown
in the following work on $Ter(G)$, not on $\POSITIONS$.

\paragraph{Syntactic conventions}
We use (possibly with subscripts or primes)
$\POSITION$ to indicate positions in $\POSITIONS$,
$s$ and $t$ for generic terms, $v$ for values in $\UNIVERSE$,
$\VECTOR{t}$ and $\VECTOR{z}$ for of terms in $Ter(G)$.
Sequences are allowed to be empty, if not specified otherwise in a side condition.
Just to make the notation more compact we will write
$\SPLUS \VECTOR{t}$
instead of
$(\SPLUS, \VECTOR{t})$ and
$\STIMES \VECTOR{t}$
for
$(\STIMES, \VECTOR{t})$.
We write $\STURN$ instead of either $\SPLUS$ or $\STIMES$, just to avoid
duplicating otherwise identical rules.
Sequences are written with no commas, and parentheses or brackets are used to group when needed.

\begin{prooftree}
  \AxiomC{$\POSITION \in \TERMINALPOSITIONS$}
  \AxiomC{$p(\POSITION) = v$}
  \LeftLabel{[Payoff]}
  \RightLabel{}
  \BinaryInfC{$\POSITION
    \TO
    v$}
\end{prooftree}

\begin{prooftree}
  \AxiomC{$succ(\POSITION) = \VECTOR{t}$}
  \AxiomC{$\TURN(\POSITION) = \PLAYER$}
  \LeftLabel{[$\PLAYER$-expand]}
  \RightLabel{$\#\VECTOR{t} \geq 1$}
  \BinaryInfC{$\POSITION
    \TO
    \SPLUS \VECTOR{t}$}
\end{prooftree}

\begin{prooftree}
  \AxiomC{$succ(\POSITION) = \VECTOR{t}$}
  \AxiomC{$\TURN(\POSITION) = \OPPONENT$}
  \LeftLabel{[$\OPPONENT$-expand]}
  \RightLabel{$\#\VECTOR{t} \geq 1$}
  \BinaryInfC{$\POSITION
    \TO
    \STIMES \VECTOR{t}$}
\end{prooftree}




\begin{prooftree}
  \AxiomC{\NOTHING}
  \LeftLabel{[$\PLAYER$-reduce]}
  \RightLabel{$v_1 \TPLUS v_2 = v$}
  \UnaryInfC{$\SPLUS \VECTOR{t} \OPENSEQUENCE v_1\ v_2 \CLOSESEQUENCE \VECTOR{z} 
    \TO
    \SPLUS \VECTOR{t} \OPENSEQUENCE \VALUE \CLOSESEQUENCE \VECTOR{z} $}
\end{prooftree}

\begin{prooftree}
  \AxiomC{\NOTHING}
  \LeftLabel{[$\OPPONENT$-reduce]}
  \RightLabel{$v_1 \TTIMES v_2 = v$}
  \UnaryInfC{$\STIMES \VECTOR{t} \OPENSEQUENCE v_1\ v_2 \CLOSESEQUENCE \VECTOR{z}
    \TO
    \STIMES \VECTOR{t} \OPENSEQUENCE \VALUE \CLOSESEQUENCE \VECTOR{z}$}
\end{prooftree}

\begin{prooftree}
  \AxiomC{\NOTHING}
  \LeftLabel{[Return]}
  \RightLabel{}
  \UnaryInfC{$\STURN \OPENSEQUENCE \VALUE \CLOSESEQUENCE \TO v$}
\end{prooftree}

\begin{prooftree}
  \AxiomC{$\WHATEVER \TO \WHATEVER'$}
  \LeftLabel{[Context]}
  \RightLabel{for all contexts $C$}
  \UnaryInfC{$C[\WHATEVER] \TOCONTEXT C[\WHATEVER']$}
\end{prooftree}

[Payoff] simply replaces a terminal position with its value in $\UNIVERSE$, by
means of the payoff function.
[$\PLAYER$-expand] and [$\OPPONENT$-expand] expand a 
position, generating its successors and keeping track of the turn, which will
be important at reduction time.
[$\PLAYER$-reduce] and [$\OPPONENT$-reduce] combine two values
into one, using $\TPLUS$ for the player and $\TTIMES$ for the
opponent. Notice that these two rules are sources of non-determinism.
[Return] unwraps a completely evaluated term containing a single value.
[Context] allows to use the other rules within nested terms (also
introducing non-determinism).
 
Notice that keeping the relation $\TOCONTEXT$
distinct from $\TO$ allows us, when needed, to see our semantics as
a {\em term rewriting system} (TRS) \cite{terese-term-rewriting}.

\begin{proposition}$\TOCONTEXT$ is strongly confluent.
\end{proposition}

\begin{proof}
For the purposes of this proof, we consider the small-step semantics
as a pure term-rewriting system, expressed in a slightly sugared
notation. The system does {\em not} need to be conditional (CTRS), since all
the rule premises can in fact be seen as structural constraints on syntactic
constructors. $\TPLUS$ and $\TTIMES$ should also be read as syntactic
constructors, with their associative properties written as rewrite rules. 
What is a variable in the rules becomes a (syntactic) variable in the
TRS; however, we will not exploit the full power of the formal system: 
reductions will only be applied to ground terms\footnote{We do not need the full
power of unification: from a programming point of view, {\em pattern matching} as used
in ML or Haskell is enough for our purposes.}.

Our TRS is trivially {\em left-} and {\em right-linear}, as no variable occurs
more than once in each side of a rule. By showing that our system is also {\em
  strongly closed}, strong confluence follows by Huet's Lemma~3.2 in \cite{confluence--huet}: ``If $\mathscr{R}$
is a left- and right-linear strongly closed term rewriting system, $\TO_{\mathscr{R}}$ is
strongly confluent''.

In order to show that the system is strongly-closed, we have to show that for
every critical pair $s, t$ there exist $s', t'$ such that
$s \TO^{*} t' \FROM^{\equiv} t$ and
$t \TO^{*} s' \FROM^{\equiv} s$
(as in \cite{confluence--huet} and \cite{terese-term-rewriting}), where
$\FROM^{\equiv}$ is the reflexive closure of $\FROM$.

The left-hand side of [$\PLAYER$-reduce] is
$\SPLUS \VECTOR{t} \OPENSEQUENCE v_1\ v_2 \CLOSESEQUENCE \VECTOR{z}$.
When this rule is used to generate a critical pair with any other rule, 
only a variable in $\VECTOR{t}$ or in $\VECTOR{z}$ can match, with the whole
left-hand side of the other rule. 
The resulting critical pair $s, t$ reaches confluence (to $s' = t'$) in one
step because redexes are non-overlapping. The same holds for [$\OPPONENT$-reduce].

The only rule pairs candidate for overlapping are [$\PLAYER$-reduce] with
itself, and [$\OPPONENT$-reduce] with itself; we only show the first one.
The only interesting case of overlapping is the term family
$\SPLUS \VECTOR{t} \OPENSEQUENCE v_1\ v_2\ v_3 \CLOSESEQUENCE \VECTOR{z}$,
generating the critical pair $s, t$. Notice that $s' \TO t'$ and vice-versa
because of the associativity of $\TPLUS$:
\[
\begin{array}{c}
\SPLUS \VECTOR{t} \OPENSEQUENCE v_1\ v_2\ v_3 \CLOSESEQUENCE \VECTOR{z}\ \ \ \ \ \\
\begin{array}{r c c c l}
& & \swarrow \ \ \ \ \ \ \ \searrow & & \\
s =&
\SPLUS \VECTOR{t} \OPENSEQUENCE (v_1 \TPLUS v_2) v_3 \CLOSESEQUENCE \VECTOR{z} &
&
\SPLUS \VECTOR{t} \OPENSEQUENCE v_1 (v_2 \TPLUS v_3) \CLOSESEQUENCE \VECTOR{z} &
= t\\
& \DOWNTO & & \DOWNTO & \\
s' = &
\SPLUS \VECTOR{t} \OPENSEQUENCE (v_1 \TPLUS v_2) \TPLUS v_3 \CLOSESEQUENCE \VECTOR{z}&
\ \ \leftrightarrows \ \ &
\SPLUS \VECTOR{t} \OPENSEQUENCE v_1 \TPLUS (v_2 \TPLUS v_3) \CLOSESEQUENCE \VECTOR{z}&
= t'\ \ \ \qed\\
\end{array}
\end{array}
\]

\end{proof}

\begin{definition}[Game tree]
Let $\rightarrowtriangle$ be the sub-rewrite system of $\TOCONTEXT$, made only
by the rules [$\PLAYER$-expand], [$\OPPONENT$-expand] and [Context]: given an
initial position $\POSITION_0 \in \POSITIONS$, the \DEF{set of game tree
 prefixes} from $\POSITION_0$ is the set 
$T_{\POSITION_{0}} = \{t\ |\ \POSITION_0 \rightarrowtriangle^{*} t\}$.
The \DEF{game tree}, if it exists, is the tree
$t_{\POSITION_{0}} \in T_{\POSITION_{0}}$ whose positions are all terminal.
\end{definition}

The game tree $t_{\POSITION_{0}}$ is well-defined: when it exists it is unique.
Actually, the TRS defining $\rightarrowtriangle^{*}$ is non-ambiguous (there is no
{\em overlap} among any reduction rules) and left-linear: such a TRS
is called {\em orthogonal}, and any orthogonal TRS is confluent \cite{terese-orthogonality}.

\begin{proposition}
$\TOCONTEXT$ is normalizing for any Nötherian game.
\end{proposition}
\begin{proof}
Let a game $G = (\SYNTAX, \ALGEBRA, \PAYOFF)$ where
$\SYNTAX = (\POSITIONS, \TURN, succ)$ and
$\ALGEBRA = (\UNIVERSE, \TPLUS, \TTIMES)$ be given.
We prove normalization by exhibiting a {\em reduction order} $<$
compatible with our rules \cite{terese-term-rewriting}. 

Let us define a \DEF{weight} function $w : \POSITIONS \TO \NATURALS$ to be a
particular instance of the higher-order function
$v_{\OVERPAYOFF} : \POSITIONS \TO \UNIVERSE$,
where
$\OVERPAYOFF(\POSITION) = 2$ for any $\POSITION \in \TERMINALPOSITIONS$ and
$\SBIGTPLUS_{i=1}^{n}x_i = \SBIGTTIMES_{i=1}^{n}x_i = 2 + \sum_{i=1}^{n}x_i$
for any $x \in \NATURALS^{*}$.
Intuitively, $w$ returns $2$ times the number of nodes in the game tree
for Nötherian games.

Let
$f : Ter(G)
     \TO
     \NATURALS$
 be:

\[
\begin{array}{l c l}
f(\POSITION) = w(\POSITION), & & \POSITION \in \POSITIONS
\\
f(v) = 1, & & v \in \UNIVERSE\\
f(\STURN \OPENSEQUENCE t_1 ... t_n \CLOSESEQUENCE) = 1 + \sum_{i=1}^{n}f(t_i) & \ \ &\\
\end{array}
\]

In the formula above and in the rest of this proof $\sum$ represents the sum
operation over $\NATURALS$.
We define our order on terms by using the interpretation $f$ on
$>_{\NATURALS}$: by definition, let
$ t_0 > t_1 \text{\ iff\ } f(t_0) >_{\NATURALS} f(t_1)$.
The order $>$ is trivially {\em stable}, as our terms do not contain variables.
$>$ is also {\em monotonic} ($f$ is increasing because
$+ : \NATURALS \times \NATURALS \TO \NATURALS$ is increasing),
{\em strict} ($>_{\NATURALS}$ is strict) and {\em well-founded}
($>_{\NATURALS}$ is well-founded). Hence, $>$ is a reduction order. 

In order to prove compatibility we show that for every rule $l \TO r$ we
have $l > r$, which by definition is equivalent to $f(l) >_{\NATURALS} f(r)$.
All equalities follow from definitions or trivial algebraic manipulations:
\begin{itemize}
\item
{}[Payoff]:
$f(\POSITION) = w(\POSITION) = \OVERPAYOFF(\POSITION) = 2 >_{\NATURALS} 1 = f(v)$.

\item
{\smallskip}[$\PLAYER$-expand], [$\OPPONENT$-expand]:
$f(\POSITION) = w(\POSITION) = 2 + \sum_{i=1}^{n}w(\POSITION_i) >_{\NATURALS}
1 + \sum_{i=1}^{n}w(\POSITION_i) = 1 + \sum_{i=1}^{n}f(t_i) = f(\STURN\ \VECTOR{t})$.

\item
{\smallskip}[$\PLAYER$-reduce], [$\OPPONENT$-reduce]:
$f(\STURN\ \VECTOR{t} \OPENSEQUENCE v_1 v_2 \CLOSESEQUENCE \VECTOR{z}) =
\sum_{i=1}^{\#\VECTOR{t}}f(t_i) + f(v_1) + f(v_2) + \sum_{i=1}^{\#\VECTOR{z}}f(z_i) =
\sum_{i=1}^{\#\VECTOR{t}}f(t_i) + 1 + 1 + \sum_{i=1}^{\#\VECTOR{z}}f(z_i) >_{\NATURALS}
\sum_{i=1}^{\#\VECTOR{t}}f(t_i) + 1 + \sum_{i=1}^{\#\VECTOR{z}}f(z_i) =
f(\STURN\ \VECTOR{t} \OPENSEQUENCE v \CLOSESEQUENCE \VECTOR{z})$.

\item
{\smallskip}[Return]:
$f(\STURN\OPENSEQUENCE v \CLOSESEQUENCE) = 1 + 1 >_{\NATURALS} 1 = f(v)$.
\ \ \ \qed

\end{itemize}
\end{proof}

Intuitively, if a term converges then its
sub-terms also converge; said otherwise if a term converges in a context, then
it must also converge in the trivial (empty) context.
This is true because of the {\em non-erasing} nature of our system, different
from, for example, the $\lambda$-calculus having actual {\em reduction steps}
\cite{terese-orthogonality}.
More formally:

\begin{lemma}[Sub-term normalization]
\label{subterm-normalization-lemma}
Given a game $G = (\SYNTAX, \ALGEBRA, \PAYOFF)$ where
$\ALGEBRA = (\UNIVERSE, \TPLUS, \TTIMES)$,
for any term $t \in Ter(G)$ and any context $C$, if
there exists $v \in \UNIVERSE$ such that
$C[t] \TOCONTEXTSTAR v$ then there exists $v' \in \UNIVERSE$ such that
$t \TOCONTEXTSTAR v'$.
\end{lemma}
\begin{proof}
By induction over the derivation length $n$ of $C[t] \TOCONTEXTSTAR v$.
We look at the possible shape of the premise of the [Context] rule,
$s \TO s'$.
\begin{itemize}
\item
Base case, $n = 1$: $C[t] \TOCONTEXT v$.
The only applicable rules are [Payoff] and [Return]:
in the case of [Payoff], $C[t] = t$; in the case of [Return],
$t = v$. In either case, $t \TOCONTEXTSTAR v = v'$.
\item
Recursive case $n \IMPLIES n + 1$: $t_0 = C[t] \TOCONTEXT t_1 \TOCONTEXTSTAR v$. The inductive
hypothesis is that for any term $t'$ and context $C'$ if
$C'[t'] \TOCONTEXTSTAR v$ in $n$ or fewer steps, then
$t' \TOCONTEXTSTAR v'$. Three cases:
\begin{itemize}
\item
  $s$ and $t$ are disjoint sub-terms within $C$. \ Since the system is
  non-erasing $t$ has not been erased, i.e. $t_1 = C'[t]$; for inductive
  hypothesis $t \TOCONTEXTSTAR v'$. 
  
\item
  $s$ contains $t$. \ $s \TO s'$ may have as its premise [Return], in which case
  $s = \STURN \OPENSEQUENCE v \CLOSESEQUENCE$ and $t = v$. Otherwise the
  premise may be
  [$\PLAYER$-Reduce] or [$\OPPONENT$-Reduce]: either $t$ is one of the values,
  or it matches one of the variables, in which case there exists a context
  $C'$ such that $t_1 = C'[t]$; then the inductive hypothesis applies.

\item
  $t$ contains $s$. \ $t = C'[s]$, hence by definition of $\TOCONTEXT$ we have
  that $t$ can turn into $C'[s'] = t'$. There exists a context $C''$ where
  $C[s] = C''[C'[s]]$, hence $t_1 = C''[C'[s']]$. By induction hypothesis
  $t' = C'[s'] \TOCONTEXTSTAR v'$. \qed
\end{itemize}
\end{itemize}
\end{proof}

Normalization and confluence justify our re-definition
of the game value $\EVALUATION$ as the transitive closure of the transition
relation $\TOCONTEXT$:
\begin{definition}[Game Value]
Let a game, an initial position $\POSITION$ and a value $\VALUE$ be
given; we say that the game value from $\POSITION$ is $\VALUE$ (and we write
$\EVALUATION(\POSITION) = v$) if and only if
$\POSITION \TOCONTEXTSTAR v$.
\end{definition}

\section{\ALPHABETA\ pruning}
\label{alpha-beta-pruning}

The \DEF{\ALPHABETA\ algorithm}
\cite{tree-prune,alpha-beta-analysis--knuth}
is a method for computing the exact value of a $min$-$max$ combinatorial game
without exhaustively visiting {\em all} game positions.

The \ALPHABETA\ algorithm is traditionally presented as a recursive
function written in imperative style (see Figure~\ref{algorithms-figure}):
the function \texttt{alpha\_beta}
analyzes a game position $\POSITION \in \POSITIONS$ with
two additional parameters, $\alpha$ and $\beta$, each one denoting
a sort of \DEF{threshold} not to be overstepped during the 
incremental computation of the value of $\POSITIONS$. Whenever the threshold is
past the evaluation of an entire subtree is aborted, as it can be proven that
it will not contribute to the result.

The correctness of \ALPHABETA\ relies on the algebraic properties of the
$min$ and $max$ functions, notably their {\em mutual} distributive laws
{--- something we can {\em not} count on under our weaker hypotheses on
  $\TPLUS$ and $\TTIMES$} \cite{alpha-beta-logic--loddo--di-cosmo,alpha-beta-pruning-under-partial-orders,phd-thesis--loddo}.
\begin{figure*}
\begin{tabular}{l|r}
{
\begin{lstlisting}[language=pseudo,numbers=left]
function $\text{alpha\_beta}$($\POSITION : \POSITIONS;\ \alpha, \beta : \TRADITIONALUNIVERSE$):$\TRADITIONALUNIVERSE\ $
  if $\POSITION \in \TERMINALPOSITIONS$ then
    return $\PAYOFF(\POSITION)$
  $\POSITION_1 ... \POSITION_n := succ(\POSITION)$ $\#\ n \ge 1$
  if $\TURN(\POSITION)$ = $\PLAYER$ then
    $v := \alpha$
    for $i$ from 1 to $n$
        and while $\beta <_\TRADITIONALUNIVERSE v$ do
      $v := min\{v, \text{alpha\_beta}(\POSITION_i, v, \beta)\}$
  else # $\TURN(\POSITION)$ = $\OPPONENT$
    $v := \beta$
    for $i$ from 1 to $n$
        and while $v <_\TRADITIONALUNIVERSE \alpha$ do
      $v := max\{v, \text{alpha\_beta}(\POSITION_i, \alpha, v)\}$
  return $v$
\end{lstlisting}
}
&
{
\begin{lstlisting}[language=pseudo,numbers=none]
function $\text{tropical}$($\POSITION : \POSITIONS;\ \alpha : \UNIVERSE$):$\UNIVERSE$
  if $\POSITION \in \TERMINALPOSITIONS$ then
    return $\PAYOFF(\POSITION)$
  $\POSITION_1 ... \POSITION_n := succ(\POSITION)$ $\#\ n \ge 1$
  if $\TURN(\POSITION)$ = $\PLAYER$ then
    $v := \alpha$
    for $i$ from 1 to $n$ do
      $\#\ \text{do not prune at }\PLAYER\text{'s level}$
      $v := v \TPLUS \text{tropical}(\POSITION_i, v)$
  else $\#\ \TURN(\POSITION)$ = $\OPPONENT$
    $v := \text{tropical}(\POSITION_1, \alpha)$ $\text{\# No }\ONE_{\UNIVERSE}$
    for $i$ from 2 to $n$
        and while $\alpha \TPLUS v \neq \alpha$ do
      $v := v \TTIMES \text{tropical}(\POSITION_i, \alpha)$
  return $v$
\end{lstlisting}
}
\end{tabular}
\caption{\label{algorithms-figure}Pruning algorithms: traditional
  \ALPHABETA\ pruning vs. tropical $\alpha$-pruning.
Notice that the tropical version has the first iteration of the second loop
unrolled, in order not to depend on the existence of a neutral element for $\TTIMES$.}
\end{figure*}
\\
Going back to our game semantics presentation we can model the \ALPHABETA's
behavior by adding four more rules --- two per player:
\begin{prooftree}
  \AxiomC{\NOTHING}
  \LeftLabel{[$\PLAYER$-\WILL]}
  \RightLabel{}
  \UnaryInfC{$\SPLUS \OPENSEQUENCE \alpha\ [\STIMES \OPENSEQUENCE \beta \ (\SPLUS \TERMS_1) \CLOSESEQUENCE\ \TERMS_2] \CLOSESEQUENCE\ \TERMS_3
    \ \TO \ 
    \SPLUS \OPENSEQUENCE \alpha\ [\STIMES \OPENSEQUENCE \beta \ (\SPLUS \OPENSEQUENCE \alpha \CLOSESEQUENCE \TERMS_1) \CLOSESEQUENCE\ \TERMS_2 ] \CLOSESEQUENCE \ \TERMS_3$}
\end{prooftree}

\begin{prooftree}
  \AxiomC{\NOTHING}
  \LeftLabel{[$\OPPONENT$-\WILL]}
  \RightLabel{}
  \UnaryInfC{$\STIMES \OPENSEQUENCE \beta\ [\SPLUS \OPENSEQUENCE \alpha \ (\STIMES \TERMS_1) \CLOSESEQUENCE\ \TERMS_2] \CLOSESEQUENCE\ \TERMS_3
    \ \TO \ 
    \STIMES \OPENSEQUENCE \beta\ [\SPLUS \OPENSEQUENCE \alpha \ (\STIMES \OPENSEQUENCE \beta \CLOSESEQUENCE \TERMS_1) \CLOSESEQUENCE\ \TERMS_2 ] \CLOSESEQUENCE \ \TERMS_3$}
\end{prooftree}

\begin{minipage}{\linewidth}
\begin{minipage}{0.5\linewidth}
\begin{prooftree}
  \AxiomC{$\alpha \TPLUS \beta = \alpha$}
  \LeftLabel{[$\PLAYER$-cut]}
  \RightLabel{}
  \UnaryInfC{$\SPLUS \OPENSEQUENCE \alpha \ (\STIMES \OPENSEQUENCE \beta \CLOSESEQUENCE \TERMS_1 ) \CLOSESEQUENCE\ \TERMS_2
              \TO
              \SPLUS \OPENSEQUENCE \alpha \CLOSESEQUENCE \TERMS_2
              $}
\end{prooftree}
\end{minipage}
\begin{minipage}{0.5\linewidth}
\begin{prooftree}
  \AxiomC{$\beta \TTIMES \alpha = \beta$}
  \LeftLabel{[$\OPPONENT$-cut]}
  \RightLabel{}
  \UnaryInfC{$\STIMES \OPENSEQUENCE \beta \ (\SPLUS \OPENSEQUENCE \alpha \CLOSESEQUENCE \TERMS_1 ) \CLOSESEQUENCE\ \TERMS_2
              \TO
              \STIMES \OPENSEQUENCE \beta \CLOSESEQUENCE \TERMS_2
              $}
\end{prooftree}
\end{minipage}
\end{minipage}
\\\\


The initialization of $v$ at line~6 should be read as 
a first ``virtual'' move of the player, whose evaluation is the value $\alpha$
inherited from an ancestor (the grandparent in an alternate-turn
game). This explains the rationale of [$\PLAYER$-\WILL]\footnote{``Will''
  should be interpreted as ``bequeath'', in the sense of leaving something as
  inheritance to a descendent.}:
whenever subtrees are nested with turns $\PLAYER$-$\OPPONENT$-$\PLAYER$,
a grandparent may cross two levels 
and ``give'' its grandchild its current accumulator as an initialization value.
Of course line~10 is the dual version for the opponent and [$\OPPONENT$-\WILL].

[$\PLAYER$-cut] and [$\OPPONENT$-cut] are a simple reformulation of the
\DEF{cut conditions} at lines 7 and 11, where the explicit order $<_{\TRADITIONALUNIVERSE}$
disappears\footnote{This is customary with lattices, when an order
 is derived from a
 {\em least-upper-bound} or {\em greatest-lower-bound} operation.}
from the condition, now expressed as an equality constraint in the rule premise: 
$\alpha \TPLUS \beta = \alpha$ represents the fact that the player would
prefer $\alpha$ over $\beta$.
Dually, $\beta \TTIMES \alpha = \beta$ means that the opponent
would prefer $\beta$ over  $\alpha$.

\begin{remark}[Non-alternate turn games]
\label{non-alternate-turn}
Notice that the cut rules can just fire in alternate-turn contexts: this choice
simplifies our exposition, but does not limit generality: see
Remarks~\ref{alternate-turn-transform} and~\ref{alternate-turn-transform-invariance}.
\end{remark}

The presence of two exactly symmetrical behaviors is quite evident in either
presentation; yet what we are interested in showing now is the fact that such
duality is quite incidental: it occurs in a natural way in actual two-player
games, yet many more search problems lend themselves to be modeled as games
despite lacking an intrinsic symmetry.
\\
\\
We can see \ALPHABETA\ as the union of two separate techniques
applied at the same time, breaking the algebraic symmetry of the
player/opponent operations: in the following we are going to eliminate the
rules [$\OPPONENT$-\WILL] and {}[$\OPPONENT$-cut], 
or equivalently to turn
\texttt{alpha\_beta} into \texttt{tropical} (see
Figure~\ref{algorithms-figure}), exploiting the weaker properties of
\DEF{tropical algebras} which only allow {\em one} threshold $\alpha$.

\section{Tropical games}
\label{tropical-games}
As we are dealing with a relatively young research topic, it is not surprising
that the formalization of tropical algebras has not yet crystallized into a
standard form. And since several details differ among the various presentations,
 we have to provide our own definition:

\begin{definition}[Tropical Algebra]
\label{tropical-algebra-definition}
An algebra $(\UNIVERSE, \TPLUS, \TTIMES)$ is called a \DEF{tropical algebra} if
it satisfies the following properties for any $a$, $b$ and $c$ in $\UNIVERSE$:
\begin{enumerate}[(i)]
\item Associativity of $\TPLUS$:
$\ a \TPLUS (b \TPLUS c) = (a \TPLUS b) \TPLUS c$
\item Associativity of $\TTIMES$:
$\ a \TTIMES (b \TTIMES c) = (a \TTIMES b) \TTIMES c$
\item Left-distributivity of $\TTIMES$ with respect to $\TPLUS$:
$\ a \TTIMES (b \TPLUS c) = (a \TTIMES b) \TPLUS (a \TTIMES c)$
\item Right-distributivity of $\TTIMES$ with respect to $\TPLUS$:
$\ (a \TPLUS b) \TTIMES c = (a \TTIMES c) \TPLUS (b \TTIMES c)$
\end{enumerate}
\end{definition}

Some particular choices of $\UNIVERSE$, $\TPLUS$ and $\TTIMES$ are
widely used: the \DEF{min-plus algebra} is obtained by defining
$\UNIVERSE \EQD \mathbb{R} \UNION \{\PINF\}$,
$a \TPLUS b \EQD min\{a, b\}$
and, a little counter-intuitively\footnote{The particular symbols used for indicating
  $\TPLUS$ and $\TTIMES$ are justified by the analogy with $+$ and
  $\cdot$ in how the distributive law works.},
$a \TTIMES b \EQD a + b$.

Since $\UNIVERSE$ and $\TTIMES$ can also be usefully instantiated in other ways,
we will not simply adopt a min-plus algebra; anyway {\em in practice} we will
also choose $\TPLUS$ to be a minimum on $\UNIVERSE$, 
which {\em in practice} will
have a total order. This seems to be the only reasonable
choice for the applications\footnote{Logic programming is an example of an
  interesting problem lending itself 
  to be interpreted as a combinatorial game on a universe with no total order
  \cite{phd-thesis--loddo,alpha-beta-logic--loddo--di-cosmo}. Anyway the
  underlying game is a symmetrical {\em inf-sup} rather than simply tropical.}
and helps to understand the idea, yet nothing in the theory depends on
the existence of the order.
Again, {\em in practice}, $\TPLUS$ will return one of its parameters, so if
needed we will always be able to trivially define a total order as
$x \leq y $~iff~$ x \TPLUS y = x$, for any $x$ and $y$ in $\UNIVERSE$.
$\TPLUS$ and $\TTIMES$ will also tend to be commutative {\em in practice},
making one of the two distributive properties trivial.

We will not make any of the supplementary hypotheses above; on the other hand,
we will require the following \DEF{rationality hypothesis}\footnote{In lattice
  theory, the rationality hypothesis is one of the \DEF{absorption identities}.}:

\begin{definition}[Rationality]
\label{rationality-definition}
Let $(\UNIVERSE, \TPLUS, \TTIMES)$ be a tropical algebra such that
$\ZERO \in \UNIVERSE$ is a neutral element for $\TPLUS$
and
$\ONE \in \UNIVERSE$ is a neutral element for $\TTIMES$\footnote{The existence
  of neutral elements is not
  strictly necessary, but it simplifies many statements and proofs; without them several
  results should be given in both ``left'' and ``right'' forms.}.
We call the algebra \DEF{rational} if, for any
$x, y, z \in \UNIVERSE$ we have $x \TPLUS (y \TTIMES x \TTIMES z) = x$.
\end{definition}

Intuitively, the opponent accumulates costs with $\TTIMES$, ``worsening'' the game
value for the player: the player will always choose just $x$ over $x$ ``worsened'' by
something else.
Notice that the notion of rationality for two-player games in Game Theory also
includes the dual condition  
$x \TTIMES (y \TPLUS x \TPLUS z) = x$; such condition
does {\em not} hold in general for tropical games.

\begin{definition}[Tropical Game, Tropical Trees]
A \DEF{tropical game}
$G = (\SYNTAX, \ALGEBRA, \PAYOFF)$
is simply a game based on a {\em rational} tropical
algebra $\ALGEBRA$.
We call \DEF{tropical trees} all the game trees of a tropical game, and
\DEF{tropical pruning} the $\alpha$-pruning of a tropical tree.
\label{bi-tropical}
A \DEF{bi-tropical game} is a tropical game 
whose dual $G^{\BOTTOM} = (\SYNTAX^{\BOTTOM}, \ALGEBRA^{\BOTTOM}, \PAYOFF)$ is
also tropical, where $\ALGEBRA^{\BOTTOM} = (\UNIVERSE, \TTIMES, \TPLUS)$
 if $\ALGEBRA = (\UNIVERSE, \TPLUS, \TTIMES)$.
\end{definition}

\subsection{Soundness of tropical pruning}
\label{pruning-soundness-proof}

\begin{proposition}[Insertion property]
\label{insertion-property}
Let $(\UNIVERSE, \TPLUS, \TTIMES)$ be a rational tropical algebra. Then
for any $x, y, \alpha, \beta \in \UNIVERSE$ we have
$\alpha \TPLUS (\beta \TTIMES x \TTIMES y) =
 \alpha \TPLUS (\beta \TTIMES (\alpha \TPLUS x) \TTIMES y)$.
\end{proposition}
\begin{proof}[Using associativity implicitly]
$\alpha \TPLUS (\beta \TTIMES (\alpha \TPLUS x) \TTIMES y) =$ \{right-distributivity\}
$\alpha \TPLUS (\beta \TTIMES ((\alpha \TTIMES y) \TPLUS (x \TTIMES y))) =$ \{left-distributivity\}
$(\alpha \TPLUS (\beta \TTIMES \alpha \TTIMES y) \TPLUS (\beta \TTIMES x \TTIMES y) =$ \{rationality\}
$\alpha \TPLUS (\beta \TTIMES x \TTIMES y)$
\qed
\end{proof}

The insertion property is the semantic counterpart of the rule [$\PLAYER$-\WILL]:
it explains why we can ``transfer'' $\alpha$ down in the tree (or more
operationally,  why we can ``start'' from the same $\alpha$ when choosing with
$\TPLUS$ two plies below), without affecting the game value.


\begin{definition}[$\PLAYER$-irrelevance]
\label{player-irrelevance-definition}
Let $(\UNIVERSE, \TPLUS, \TTIMES)$ be a rational tropical algebra, and let
$\alpha, \beta \in \UNIVERSE$. Then we call $x \in \UNIVERSE$ \DEF{$\PLAYER$-irrelevant
with respect to $\alpha$ and $\beta$} if
$\alpha \TPLUS (\beta \TTIMES x) = \alpha$.
\end{definition}

Intuitively, as the value of an opponent-level tree, $x$ can't affect the
value of the game because the player will not give the opponent the
opportunity to be in that situation: in other word, the current
optimal move for the player doesn't change because of $x$.

\begin{lemma}[$\PLAYER$-irrelevance]
\label{player-irrelevance-lemma}
Let $(\UNIVERSE, \TPLUS, \TTIMES)$ be a rational tropical algebra, and
$\alpha, \beta \in \UNIVERSE$. If $\alpha \TPLUS \beta = \alpha$ then
any $x \in \UNIVERSE$ is $\PLAYER$-irrelevant with respect to $\alpha$ and $\beta$.
\end{lemma}
\begin{proof}
$\alpha \TPLUS (\beta \TTIMES x) = $\ \{hypothesis\}
$(\alpha \TPLUS \beta) \TPLUS (\beta \TTIMES x) = $\ \{associativity\}
$\alpha \TPLUS (\beta \TPLUS (\beta \TTIMES x)) = $\ \{rationality\}
$\alpha \TPLUS \beta = $\ \{hypothesis\}
$\alpha$
\qed
\end{proof}

\begin{definition}[Simulation]
\label{simulation-definition}
Given a tropical game,
we say that \DEF{a term $t'$ simulates a term $t$}, and we write $t \leq t'$, if
$t \TOCONTEXTSTAR v \IMPLIES t' \TOCONTEXTSTAR v$.
\end{definition}

\begin{lemma}[Tropical $\PLAYER$-\WILL\ simulation]
\label{alpha-down-propagation-lemma}
Given a tropical game $G = (\SYNTAX, \ALGEBRA, \PAYOFF)$
where
$\ALGEBRA = (\UNIVERSE, \TPLUS, \TTIMES)$, for any term sequence
$\alpha, \beta \in \UNIVERSE$,
$\VECTOR{t_0}, \VECTOR{t_1}, \VECTOR{t_2} \in Ter(G)^{*}$
$$
\SPLUS \OPENSEQUENCE \alpha \ [\STIMES \OPENSEQUENCE \beta \ (\SPLUS \TERMS_0) \CLOSESEQUENCE\ \TERMS_1] \CLOSESEQUENCE\ \TERMS_2
\ \le\ 
\SPLUS \OPENSEQUENCE \alpha \ [\STIMES \OPENSEQUENCE \beta \ (\SPLUS \OPENSEQUENCE \alpha \CLOSESEQUENCE \ \TERMS_0) \CLOSESEQUENCE\ \TERMS_1] \CLOSESEQUENCE \ \TERMS_2
$$
\end{lemma}
\begin{proof}
By the Sub-term normalization Lemma, if $t$ converges there will exist some value sequences
$\VALUES_0, \VALUES_1, \VALUES_2 \in \UNIVERSE$ such that
$\TERMS_0 \TOSTAR \VALUES_0$,
$\TERMS_1 \TOSTAR \VALUES_1$,
$\TERMS_2 \TOSTAR \VALUES_2$;
let us call
$\VALUE_0$ the result of $\BIGTPLUSSMALLER \VALUES_0$,
$\VALUE_1$ the result of $\BIGTTIMESSMALLER \VALUES_1$ and
$\VALUE_2$ the result of $\BIGTPLUSSMALLER \VALUES_2$.
Then,
$$
\begin{array}{ c c c }
\SPLUS \OPENSEQUENCE \alpha\ [\STIMES \OPENSEQUENCE \beta\ (\SPLUS \TERMS_0) \CLOSESEQUENCE\ \TERMS_1] \CLOSESEQUENCE\ \TERMS_2 &
&
\SPLUS \OPENSEQUENCE \alpha\ [\STIMES \OPENSEQUENCE \beta\ (\SPLUS \OPENSEQUENCE \alpha \CLOSESEQUENCE\ \TERMS_0) \CLOSESEQUENCE\ \TERMS_1] \CLOSESEQUENCE \ \TERMS_2
\\
\DOWNTOSTAR &
&
\DOWNTOSTAR
\\
\SPLUS \OPENSEQUENCE \alpha\ [\STIMES \OPENSEQUENCE \beta\ \VALUE_0\ \VALUE_1 \CLOSESEQUENCE]\ \VALUE_2 \CLOSESEQUENCE &
&
\SPLUS \OPENSEQUENCE \alpha\ [\STIMES \OPENSEQUENCE \beta\ (\alpha \TPLUS \VALUE_0)\ \VALUE_1 \CLOSESEQUENCE]\ \VALUE_2 \CLOSESEQUENCE
\\
\DOWNTOSTAR &
\text{\small \{Insertion\}} &
\DOWNTOSTAR
\\
\alpha \TPLUS [\beta \TTIMES \VALUE_0 \TTIMES \VALUE_1] \TPLUS \VALUE_2 &
= &
\alpha \TPLUS [\beta \TTIMES (\alpha \TPLUS \VALUE_0) \TTIMES \VALUE_1] \TPLUS \VALUE_2 
\end{array}
$$
In the reductions above we implicitly assume that some sequences are non-empty;
the proof trivially generalizes to empty $\VECTOR{t_1}$ and $\VECTOR{t_2}$  by
using neutral elements. \qed
\end{proof}



\begin{lemma}[Tropical cut simulation]
Given a tropical game $G = (\SYNTAX, \ALGEBRA, \PAYOFF)$
where
$\ALGEBRA = (\UNIVERSE, \TPLUS, \TTIMES)$, for any term sequence
$\alpha, \beta \in \UNIVERSE$,
$\VECTOR{t_0}, \VECTOR{t_1} \in Ter(G)^{*}$ we have that
if $\alpha \TPLUS \beta = \alpha$, then
$\SPLUS \OPENSEQUENCE \alpha\ (\STIMES \OPENSEQUENCE \beta \CLOSESEQUENCE \ \TERMS_0) \CLOSESEQUENCE \ \TERMS_1
\ \le\ 
\SPLUS \OPENSEQUENCE \alpha \CLOSESEQUENCE \ \TERMS_1$.
\end{lemma}
\begin{proof}
Just like Lemma~\ref{alpha-down-propagation-lemma}, with $\PLAYER$-irrelevance
at the end. \qed

\end{proof}

\begin{theorem}[Tropical rule soundness]
Adding the rules [$\PLAYER$-\WILL] and [$\PLAYER$-cut] ``does not alter
semantics'', i.e. if a term $t$ converges to a value $v$ in a system without the
two new rules, it is guaranteed to have a reduction sequence converging to $v$
also in the extended system.
\qed
\end{theorem}


\section{Choose-How-To-Divide and Conquer}
\label{choose-how-to-divide-and-conquer}
According to the classical Divide and Conquer technique a
problem can be divided into subproblems, each of which will be solved
recursively until a minimal-size instance is found; sub-solutions will
then be recomposed.

In the traditional Divide and Conquer style, each division choice is final: it is
considered taken once and for all, and cannot be undone. By contrast we
present an alternative model
based on tropical games.
In the \DEF{Choose-How-To-Divide and Conquer} style we
work with non-deterministic choices in a solution space,
using a quality criterion to be optimized and some way of
``combining'' sub-solutions.

Of course many nondeterministic algorithms can be expressed this way: the
challenge is finding a suitable mapping to the tropical game concepts, in term
of both syntax and semantics (with the required properties). {The
problem must have both a suitable {\em syntactic} structure, and a {\em
  semantic} structure with the required properties.}

The action of {\em choosing a division} corresponds to a player node where the
$\TPLUS$ function (typically a minimization) returns the ``best'' option; the
points where {\em sub-solutions have their cost accumulated} (often something
similar to a sum, intuitively ``opposed'' to $\TPLUS$) become opponent nodes
where $\TTIMES$ combines the values of a subtree sequence into a single result.

Tropical trees have the desirable property of supporting $\alpha$-pruning,
with the potential of significantly cutting down the search space. 
{The more [$\PLAYER$-\WILL] and [$\PLAYER$-cut] can fire, the more
pruning is profitable: hence the problem should be represented as a
tropical game having {\em alternate turns} and
{\em branching factor greater than $2$ for $\OPPONENT$} at least ``often
enough''.}
\\
{Search problems abound in Artificial Intelligence, and in particular we
  suspect that more symbolic computation problems than one might expect can be
  modeled this way. We now proceed to show an unusual application of
  {\em Choose-How-To-Divide and Conquer}.}



\subsection{Parsing as a tropical game}
Let $\GRAMMAR$ be a given context-free grammar, defined over an alphabet of
terminals $A \ni a$ and nonterminals $N \ni X$.
For simplicity\footnote{Such restrictions can be lifted at the cost of some
  complexity, but supporting a larger class of grammars would be quite
  inessential for our demonstrative purposes.} let it have no 
  $\epsilon$-production, nor any productions with two consecutive nonterminals
  or a single nonterminal alone in the right-hand side.
Right-hand sides will hence be of the form
$[a_1] X_1 a_2 X_2 ... a_n X_n [a_{n+1}]$, with $n \ge 0$ and at least one $a_i$.
Given a string of terminals $s \in A^{+}$ our problem is finding the ``best''
  parse tree of $s$ in $\GRAMMAR$; when $s$ contains some errors our ``best'' solution
  just ends up being {\em the least wrong}, according to some metric; just to
  keep things simple in this example out metric to minimize will be the total
  size of the substrings which cannot be matched, in terminals.
Sometimes we may wish to have the set of {\em all} best parses, instead of
being content with just one optimal solution.

\subsubsection{Syntax.}
The set of game positions is defined as 
$\POSITIONS = (A^{*} \times N) \DISJOINTUNION (A^{*} \times N)^{*}$, and the
turn function is
$\TURN(s, X) = \PLAYER$, $\TURN ((s_1, X_1) ... (s_k, X_k)) = \OPPONENT$. These
definitions become easy to understand once the successor function $succ$ is examined.

A player position has the form $\POSITION_{\PLAYER} = (s, X)$, since the player
has to parse a string $s$ with a nonterminal $X$. It has to choose a production
$X ::= [a_1] X_1 a_2 X_2 ...$\\
$ a_n X_n [a_{n+1}]$,
 and match the 
terminals $a_i$ with the terminals in $s$, in the right order. 
Each possible match of {\em all} terminals, for each production of
$X$, is a valid player move generating {\em strictly smaller} subproblems for the
opponent: the nonterminals $X_i$ ``in between'' the matched terminals will have
to be matched to substrings of $s$ in the opponent position
$\POSITION_{\OPPONENT} = (s_1, X_1) ... (s_1, X_n)$, for some $n \ge 0$. If no
match exists with any production then $\POSITION_{\PLAYER}$ is terminal.

In an opponent position $\POSITION_{\OPPONENT} = (s_1, X_1) ... (s_1, X_n)$
the opponent has always exactly $n$ moves: the opponent will give the player
each pair $(s_i, X_i)$ to solve ``one at the time''. For this reason the
successor of an opponent position is equal to the position itself: it is the
sequence of the elements of $\POSITION_{\OPPONENT}$, itself a
sequence. An opponent position $\POSITION_{\OPPONENT}$ is terminal when it is empty.

Figure~\ref{grammar-and-parse-tree-figure} contains a practical example.

\begin{figure}
\begin{minipage}[b]{0.45\linewidth}
$E ::= \textbf{n} \\
 E ::= \textbf{v} \\
 E ::= \textbf{( } E \textbf{ )} \\
 E ::= \textbf{let } \textbf{v = } E \textbf{ in } E\\
 E ::= \textbf{if } E \textbf{ then } E \textbf{ else } E\\
 E ::= E \textbf{ = } E \\
 E ::= E \textbf{ $+$ } E \\
 E ::= E \textbf{ \textasteriskcentered } E$
\end{minipage}
\begin{minipage}[b]{0.54\linewidth}
\includegraphics[height=3cm]{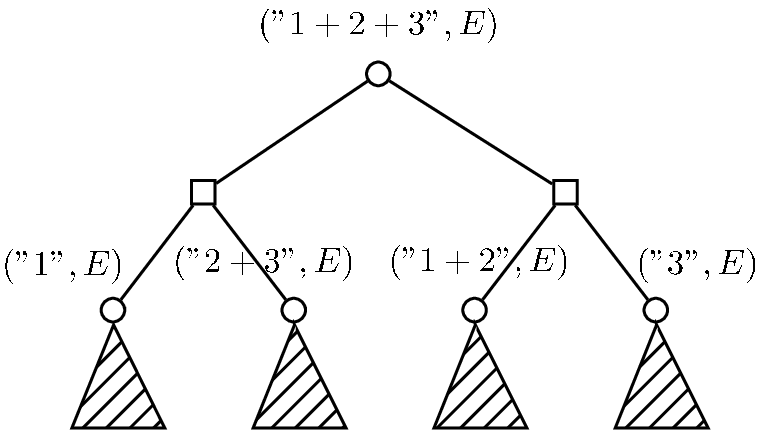}
\end{minipage}
\caption{\label{grammar-and-parse-tree-figure}
  We use the simple grammar $\GRAMMAR$ above, with an intentionally high level
  of ambiguity, to parse the string $\texttt{"1 + 2 + 3"}$
  with $E$ as the start symbol. Circles represent $\SPLUS$ nodes, squares are for $\STIMES$.}
\end{figure}

\subsubsection{Semantics.}
We use a {\em min-plus} algebra for $\ALGEBRA = (\UNIVERSE, \TPLUS, \TTIMES)$: we
simply define $\UNIVERSE \EQD \NATURALS$; we
take $\TPLUS \EQD min$, since we want as few errors as possible; and finally
$\TTIMES \EQD +$: the number of total errors in the parse tree is equal to the sum of the number
of errors in all subtrees.

The payoff $\PAYOFF(\POSITION)$ is defined as the length in characters of the input string for
player positions (notice that the payoff is only defined on {\em terminal}
positions, so such a length is actually the number of unmatched characters), and
zero for opponent positions
(if $\POSITION_{\OPPONENT} = \OPENSEQUENCE\CLOSESEQUENCE$ then there are no
errors to accumulate: at the level above, the player matched {\em the whole} string):
 $p(s, X) \EQD \#s$,
$p(\OPENSEQUENCE\CLOSESEQUENCE) \EQD 0$.

\subsubsection{Experiments}
We implemented a prototype system\footnote{The prototype is freely available
  under the GNU GPL license
at the address\\\url{http://www-lipn.univ-paris13.fr/~loddo/aisc-2010}.} in ML supporting the grammar of
Figure~\ref{grammar-and-parse-tree-figure},
which can be configured to do a simple
exhaustive search or perform tropical $\alpha$-pruning. The prototype supports two policies:
\DEF{first-minimal (henceforth FM)} searches for only one optimal strategy at
  $\PLAYER$'s levels, and
  \DEF{all-minimals (henceforth AM)} generates a sequence of strategies with {\em non-increasing}
  cost.
\\
Just as illustrative examples, we proceed to show our system behavior on
some input strings belonging to different categories.

{\em Non-ambiguous input:}
the input string
\texttt{"let x = 42 in x + if 84=42 then 55 else 77"}
is parsable in a unique way, so the FM policy is clearly the right choice.
Compared to an exhaustive search the $\alpha$-pruning FM version avoids $98\%$
of the recursive calls ($460$ vs $28473$) and its completion time is $4\%$.
By setting the policy to AM
the number of recursive call grows a little, from $460$
to $671$ (still avoiding $97\%$ of the calls).

{\em Ambiguous input:}
with the input string \texttt{"let x = 84 = 42 = 21 in 1 + 2 * 3"},
which is parsable in several ways,
the the $\alpha$-pruning FM version avoids $99\%$ of the recursive
calls ($260$ vs $61980$), and the run time is $1\%$ of the
exhaustive-search version time.
The $\alpha$-pruning AM version still avoids $96\%$ of the
recursive calls ($2148$ vs $61980$), and its run time is $3\%$.

{\em ``Wrong'' input:}
with the input string \texttt{"if if if true then true else
false then 10 else (1+(2+)+3)"}, containing errors,
the $\alpha$-pruning FM version avoids $98\%$ of the recursive calls
($9640$ vs $494344$) and its run time is $3\%$, while the AM version
avoids $97\%$ of the recursive calls ($13820$ vs $494344$); the AM version's
run time is reduced to $3\%$.
The best strategy has value $6$, corresponding
to the size of the substring \texttt{"if true"} (blanks are
not counted) added to the size ($0$) of the empty substring delimited
by the tokens \texttt{"+"} and
\texttt{")"}. The $\alpha$-pruning algorithm has
\DEF{localized} errors, guessing that the user should fix her
string by replacing \texttt{"if true"} with something correct and
writing something correct between \texttt{"+"} and \texttt{")"} ---
having the size of the unmatched substrings as the payoff function yields
this ``smallest-incorrect-string'' heuristic.
{Of course other more elaborate criteria are also possible, such as ``minimum number
of errors''.}

{\em Memoization:}
on a completely orthogonal axis, the implementation may be configured to
perform \DEF{memoization}: when memoization is turned on all the already solved
positions are cached, so that they are not computed more than
once.
%
%
We have compared a memoizing version of our tropical-$\alpha$-pruning parser
with a memoizing version performing exhaustive search.
In the first case above, the string
\texttt{"let x = 42 in x + if 84=42 then 55 else 77"}
 is now parsed with $131$ calls
instead of $460$, again saving $98\%$ of the calls ($131$ vs $7295$)
and cutting the run time to $1\%$.
\texttt{"let x = 84 = 42 = 21 in 1 + 2 {*} 3"} is now parsed with
$72$ calls instead of $260$, avoiding $99\%$ of the calls ($72$
vs $14443$) and reducing the run time to $7\%$.
The string
\texttt{"if if if true then true else false then 10 else (1+(2+)+3)"}
is parsed with $1206$ calls instead of $9640$, avoiding $96\%$ of
calls ($1206$ vs $36575$) and cutting the completion time to $10\%$.

At least in our small test cases, tropical $\alpha$-pruning and memoization
work well together: enabling either one does not significantly lessen the efficacy of
the other.

\section{Conclusions and future work}
\label{conclusions-and-future-work}
We have introduced and formally proved correct \DEF{tropical $\alpha$-pruning},
a variant of \ALPHABETA-pruning applicable to the \DEF{tropical games}
underlying \DEF{Choose-How-To-Divide and Conquer} problems.
As a practical example of the technique we have shown how the problem of
approximated parsing
and error localization
 can be modeled as a game, and how our pruning technique
can dramatically improve its efficiency; yet an asymptotic measure of the
visited node reduction would be a worthy development.

We suspect that many more problems can be formalized as tropical games, and
the problem of parsing itself can also definitely be attacked in a more general
way, lifting our restrictions on the grammar; tropical parsing might prove to
be particularly suitable for natural language problems, with their inherent
ambiguity.

The correctness and efficiency of
\DEF{parallel} tropical $\alpha$-pruning implementations would be particularly interesting to study.
\subsubsection*{Acknowledgments}
Christophe Fouqueré first recognized tropical algebras in the properties
required by our formalization.


\nocite{alpha-beta-pruning-under-partial-orders}

\bibliographystyle{splncs}
\bibliography{aisc-2010--loddo-saiu}

\begin{thebibliography}{1}

\bibitem{tree-prune}
Hart, T.P., Edwards, D.J.:
\newblock The tree prune ({TP}) algorithm.
\newblock Artificial Intelligence Project Memo~30, Massachusetts Institute of
  Technology, Cambridge, Massachusetts (1961)

\bibitem{alpha-beta-analysis--knuth}
Knuth, D.E., Moore, R.W.:
\newblock An analysis of alpha-beta pruning.
\newblock Artificial Intelligence \textbf{6} (1975)  293--326

\bibitem{phd-thesis--loddo}
Loddo, J.V.:
\newblock {G}\'en\'eralisation des Jeux Combinatoires et Applications aux
  Langages Logiques.
\newblock PhD thesis, Université Paris VII (2002)

\bibitem{alpha-beta-logic--loddo--di-cosmo}
Loddo, J.V., Cosmo, R.D.:
\newblock Playing logic programs with the alpha-beta algorithm.
\newblock In: Logic for Programming and Automated Reasoning (LPAR). Number 1955
  in LNCS, Springer (2000)  207--224

\bibitem{alpha-beta-pruning-under-partial-orders}
Ginsberg, M.L., Jaffray, A.:
\newblock Alpha-beta pruning under partial orders.
\newblock In: In Games of No Chance II. (2001)

\bibitem{terese-term-rewriting}
Klop, J.W., de~Vrijer, R.:
\newblock First-order term rewriting systems.
\newblock In Terese, ed.: Term Rewriting Systems.
\newblock Cambridge Universisty Press (2003)  24--59

\bibitem{confluence--huet}
Huet, G.:
\newblock Confluent reductions: Abstract properties and applications to term
  rewriting systems: Abstract properties and applications to term rewriting
  systems.
\newblock J. ACM \textbf{27}(4) (1980)  797--821

\bibitem{terese-orthogonality}
Klop, J.W., Oostrom, V.V., de~Vrijer, R.:
\newblock Orthogonality.
\newblock In Terese, ed.: Term Rewriting Systems.
\newblock Cambridge Universisty Press (2003)  88--148

\end{thebibliography}

\end{document}